\documentclass[letterpaper, 10 pt, conference]{ieeeconf}  % Comment this line out if you need a4paper

\IEEEoverridecommandlockouts                              % This command is only needed if 
\usepackage{graphics} % for pdf, bitmapped graphics files
\usepackage{epsfig} % for postscript graphics files
\usepackage{mathptmx} % assumes new font selection scheme installed
\usepackage{times} % assumes new font selection scheme installed
\usepackage{amsmath} % assumes amsmath package installed
\usepackage{amssymb}  % assumes amsmath package installed
\usepackage{algorithm,algorithmic}
\usepackage{booktabs}

\newtheorem{theorem}{Theorem}

\newtheorem{assumption}{Assumption}

\usepackage{balance}
\title{\LARGE \bf
Stochastic Multi-armed Bandits with Non-stationary Rewards Generated by a Linear Dynamical System
}

\author{Jonathan Gornet, Mehdi Hosseinzadeh \IEEEmembership{Member, IEEE} and Bruno Sinopoli \IEEEmembership{Fellow, IEEE}%
\thanks{The authors are with the Department of Electrical and Systems Engineering, Washington University in St. Louis, St. Louis, MO 63130, USA (email: jonathan.gornet@wustl.edu; mhosseinzadeh@wustl.edu; bsinopoli@wustl.edu). This material was based upon work supported by awards ARO W911NF20S0009.}}

\begin{document}

\maketitle
\thispagestyle{empty}
\pagestyle{empty}

%%%%%%%%%%%%%%%%%%%%%%%%%%%%%%%%%%%%%%%%%%%%%%%%%%%%%%%%%%%%%%%%%%%%%%%%%%%%%%%%
\begin{abstract}
    The stochastic multi-armed bandit has provided a framework for studying decision-making in unknown environments. We propose a variant of the stochastic multi-armed bandit where the rewards are sampled from a stochastic linear dynamical system. The proposed strategy for this stochastic multi-armed bandit variant is to learn a model of the dynamical system while choosing the optimal action based on the learned model. Motivated by mathematical finance areas such as Intertemporal Capital Asset Pricing Model proposed by Merton and Stochastic Portfolio Theory proposed by Fernholz that both model asset returns with stochastic differential equations, this strategy is applied to quantitative finance as a high-frequency trading strategy, where the goal is to maximize returns within a time period.
\end{abstract}

\section{Introduction}

The stochastic multi-armed bandit (MAB) problem, proposed by Thompson in 1933 \cite{thompson1933likelihood}, has provided a powerful modeling framework to investigate a large class of decision making problems. In MAB, a learner interacts with the environment where in each interaction, called a round, the learner chooses an action and receives a reward. The performance of policies is usually evaluated as the expectation of the cumulative difference between chosen and optimal actions, and defined as regret. In its basic formulation rewards are sampled from a stationary distribution, a very popular algorithm, called the \textit{Upper Confidence Bound} (UCB) algorithm~\cite{agrawal1995sample}, guarantees logarithmic growth of regret~\cite{auer2002finite}. 

MAB has seen different applications in several areas such as machine learning, dynamic pricing, and portfolio management. In machine learning, the MAB formulation can be used to find a set of hyperparameters to increase the performance of the learning process \cite{li2017hyperband}. This has been extended further by applying MAB to algorithm selection, where a learner searches for a high-performing algorithm to use for training \cite{gagliolo2010algorithm}. For dynamic pricing, when selling a set of products, the price needs to be set according to the current demand in order to maximize profit. To find the optimal pricing, MAB is used where the model for demand is based on a low-dimensional demand model \cite{mueller2019low} or a differential equation \cite{agrawal2021dynamic}. In portfolio selection, MAB formulation is natural in the case where the manager creates a portfolio with multiple assets \cite{shen2015portfolio,huo2017risk}. 

In the examples mentioned above, the environments that the learner interacts with has time correlations and correlations between the rewards for each action. Hyperparameter optimization can be viewed as optimizing a cost function with correlated decision variables \cite{shahriari2015taking} and the training process is dynamic \cite{li2017hyperband}. In dynamic pricing, the pricing for each product can be correlated with each other based on a low dimensional demand model \cite{mueller2019low} and the demand changes over time \cite{agrawal2021dynamic}. As for portfolio management, equations used in Intertemporal Capital Asset Pricing Model (ICAPM) proposed by Merton \cite{merton1973intertemporal} and Stochastic Portfolio Theory proposed by Fernholz \cite{fernholz1982stochastic} model asset returns using stochastic differential equations. The examples above motivate the need to investigate non-stationary stochastic MAB where the rewards are sampled from a stochastic dynamical system. Here the reward can be expressed as the inner product between the action vector and a dynamic unknown parameter vector. Previous work in non-stationary MAB studies the case where the change in magnitude of the unknown parameter vector is bounded~\cite{cheung2019learning,javanmard2017perishability} or when such vector is is sampled from a predefined set~\cite{qin2022non,Sequential_Stochastic}. In both cases the unknown parameter vector changes the distribution of the reward, making the original stochastic MAB formulation non-stationary. To the best of our knowledge, the bulk of papers in non-stationary MAB focus on the \textit{piece-wise stationary} case, where the distributions are stationary within set intervals. Papers for the \textit{piece-wise stationary} case focus on remembering-vs-forgetting trade-off using either discounting or sliding windows \cite{garivier2008upper} to forget early recorded rewards or detecting the change in the distribution to decide when to restart the learning process \cite{hartland2007change,liu2018change, cao2019nearly,mellor2013thompson}. Other papers in non-stationary MAB bound the cumulative change of the reward mean in the form of variational budget $V_T$ \cite{wei2021nonstationary,besbes2014stochastic}. However, the rewards sampled from a stochastic dynamical system change in a very different manner, making previous work not well suited. In this paper we tackle the stochastic multi-armed bandit problem when the rewards are sampled from an unknown stochastic linear dynamical system driven by Gaussian noise. Since the rewards are now dynamic, the paper introduces a methodology that focuses on finding the optimal decision while learning a model of the system. We will leverage and adapt results in~\cite{tsiamis2019finite} to learn the linear model. We then use such model to design a policy based on reward predictions. To illustrate the concept, we apply the algorithm to a simple trading example.

The framework of the paper is as follows: Section \ref{sec:Problem_Framework} reviews the stochastic multi-armed bandit and introduces a variant where the rewards are sampled from a linear dynamical system. In section \ref{sec:Modeling_the_System}, a methodology to model and predict rewards is presented. Section \ref{sec:Bandit_Strategy} then uses the model to develop a strategy to maximize cumulative reward over a horizon. Section \ref{sec:Regret} performs regret analysis and provides a theoretical upper bound for the regret of the proposed algorithm. Section \ref{sec:Numerical_Simulation} shows an application of the proposed algorithm to a simple high frequency trading example. Finally, section \ref{sec:Conclusion} provides conclusions and future directions.

\noindent \textbf{Notation}. To denote the transpose of a matrix, the notation $^\top$ is used. For norms, $||\cdot||_2$ is the $\ell_2$-norm norm for vectors. The trace of a matrix is denoted with $\mbox{Tr}(\cdot)$. For a normal distribution, notation $\mathcal{N}(M,S)$ is used, where $M$ is the mean and $S$ is the covariance of the distribution. Inner product is $\langle a,b \rangle \triangleq a^\top b$ where $a, b \in \mathbb{R}^d$. The term $\mathcal{O}(\cdot)$ is big-O notation. For the equality 
    \begin{equation}
        g(n) = \mathcal{O}(f(n)), \nonumber
    \end{equation}
    this implies that for some constant $M > 0$ and $n_0 \geq 0$, $g(n)\leq Mf(n)$ for all $n>n_0$ \cite{10.1145/1008328.1008329}. 

\section{Problem Framework}\label{sec:Problem_Framework}

Suppose that for $k$ given actions $c_a \in \mathcal{A} \subset \mathbb{R}^{d}$, the reward $X_t$ is sampled from the following stochastic linear dynamical system
\begin{equation}\label{eq:Linear_System}
    \begin{cases}
        z_{t+1} & = \Gamma z_t + \xi_t, z_0 \sim \mathcal{N}(\mu_0,\Sigma_0) \\
        \theta_t & = C_\theta z_t + \phi_t \\
        X_t & = \langle c_a ,z_t \rangle + \mu_a + \eta_t
    \end{cases}, 
\end{equation}
where the unknown parameter vector $z_t \in \mathbb{R}^d$ is the state of the system. For each round $t = 1,\dots,n$, $n > d$, the learner observes the reward $X_t \in \mathbb{R}$ based on the chosen action and the context $\theta_t \in \mathbb{R}^m$. The context $\theta_t$ is a value that the learner always observes and its observation matrix $C_\theta$ is constant. The processes $\xi_t \in \mathbb{R}^d$, $\phi_t \in \mathbb{R}^m$, and $\eta_t \in \mathbb{R}$ are i.i.d. normally distributed, i.e. $\xi_t \sim \mathcal{N}(\mu_\xi ,Q)$, $\phi_t \sim \mathcal{N}(\mu_\phi,R_\phi)$, and $\eta_t \sim \mathcal{N}(0,\sigma_\eta)$. The matrices $\Gamma$, $C_\theta$, $Q$, $R_\phi$, $\Sigma_0$, vectors $\mu_\xi$, $\mu_\phi$, $\mu_0$, and scalars $\sigma_\eta$, $\mu_a$ ($a = 1,\dots,k$) are assumed to be unknown. The dimension $d$ is unknown, but dimension $m$ is known as it is the dimension of the context. For notation, given that there are $k$ vectors $c_a \in \mathcal{A}$, we denote $a \in \{1,\dots,k\}$ to be which vector $c_a$ is chosen. The system has the following assumptions

\begin{assumption}\label{assum:observable}
    The matrix pair $(\Gamma ,C_\theta)$ is observable. The matrix $Q$ is positive definite.
\end{assumption}

\begin{assumption}
    The matrix $\Gamma $ is Schur, i.e. $\rho(\Gamma) < 1$. 
\end{assumption}

% It is assumed that the learner knows that the rewards and contexts are sampled from a stochastic linear dynamical system.

The goal of the learner is to maximize cumulative reward over a finite time horizon $n$. To prove the performance of the learning strategy, regret analysis is used \cite{lattimore2020bandit}. Regret is defined as the cumulative, over all rounds, expected difference between the highest reward (denoted as $X_t^*$) and the reward for the chosen action at time $t$, i.e. 
\begin{equation}\label{eq:pseudo-regret}
    R_n = \sum_{t = 1}^n \mathbb{E}[X_t^* - X_t]. 
\end{equation}

\section{Modeling the System from Data}\label{sec:Modeling_the_System}

If the learner knew \eqref{eq:Linear_System}, then the Kalman filter could be used to predict the state $z_t$ and consequently the reward $X_t$ for each action $a \in \{1,\dots,k\}$:
\begin{equation}\label{eq:Kalman_Filter}
    \begin{array}{ll}
        \hat{z}_{t+1|t} & = \Gamma \hat{z}_{t|t} + \mu_\xi, \quad P_{t+1|t} = \Gamma P_{t|t} \Gamma^\top + Q \\
        K_t & = P_{t|t-1} C_\theta^\top (C_\theta P_{t|t-1}C_\theta + R_\phi)^{-1} \\
        \hat{z}_{t|t} & = \hat{z}_{t|t-1} + K_t (\theta_t - C_\theta \hat{z}_{t|t-1} - \mu_\phi), \\
        P_{t|t} & = P_{t|t-1} - K_t C_\theta P_{t|t-1}, \\
        \hat{X}_{t|t-1} & = \langle c_a,\hat{z}_{t|t-1}\rangle + \mu_a
    \end{array}, 
\end{equation}
where $\hat{z}_{t|t} \triangleq \mathbb{E}[z_t|\mathcal{F}_t]$ and $\mathcal{F}_t$ is the sigma algebra generated by previous contexts $\theta_0,\dots,\theta_t$. Since the Kalman gain matrix $K_t$ converges thanks to assumption \ref{assum:observable}, then using the steady-state Kalman filter is reasonable where the prediction of the state $\hat{z}_{t+1|t}$ and the estimate of the state $\hat{z}_{t|t}$ can be combined into one equation
\begin{equation}\label{eq:Steady_Kalman_Filter}
    \begin{array}{ll}
        \hat{z}_{t+1} & = \Gamma \hat{z}_{t} + \mu_\xi + \Gamma K (\theta_t - C_\theta \hat{z}_{t} - \mu_\phi)\\
        \hat{X}_t & = \langle c_a,\hat{z}_{t}\rangle + \mu_a
    \end{array}, 
\end{equation}
\begin{align}
    K & = PC_\theta^\top (C_\theta PC_\theta^\top + R_\phi)^{-1}, \nonumber \\
    P & = \Gamma  P\Gamma ^\top + Q - \Gamma PC_\theta^\top (C_\theta PC_\theta^\top + R_\phi)^{-1} C_\theta P\Gamma^\top,\nonumber \\
    \hat{z}_t & \triangleq \hat{z}_{t|t-1}, \nonumber \\
    \hat{X}_t & \triangleq \hat{X}_{t|t-1}. \nonumber
\end{align}

Since using the steady-state Kalman filter prediction $\hat{X}_t$ provides a good prediction of $X_t$, then learning the steady-state Kalman filter will intuitively provide a good prediction of the reward $X_t$ for each action $a \in \{1,\dots,k\}$. Therefore, to learn the steady-state Kalman filter, a variation of \cite{tsiamis2019finite} is used. Let $s> 0$ be the horizon length of how far we look into the past. We define a matrix $G_a$ for each $a \in \{1,\dots,k\}$ and a vector $\Theta_t$ below
\begin{align}
    G_{a} & \triangleq \left[\begin{matrix}
        c_a^\top (\Gamma -\Gamma K C_\theta)^{s-1} \Gamma K & \dots \end{matrix}\right.  \nonumber \\
        & \quad\quad \left.\begin{matrix}
             c_a^\top\Gamma K & \sum_{\tau = 1}^s \langle c_a, \Gamma^\tau (\mu_\xi - \mu_\phi) \rangle + \mu_a
        \end{matrix}\right] \in \mathbb{R}^{1 \times (ms+1)}, \nonumber \\
    \Theta_{t} & \triangleq \begin{bmatrix}
        \theta_{t-s}^\top & \dots & \theta_{t-1}^\top & 1
    \end{bmatrix}^\top \in \mathbb{R}^{ms+1 \times 1} .
\end{align}

Using $G_a$ and $\Theta_t$ defined above, it can be shown that the reward $X_t$ has the following expression 
\begin{equation}
    X_{t} = G_{a}\Theta_{t} + \langle c_a, (\Gamma -\Gamma K C_\theta)^s \hat{z}_{t-s}\rangle +  \varepsilon_{a;t}, \label{eq:Matrix_Form}
\end{equation}
where
\begin{align}
    \varepsilon_{a;t} & \triangleq X_t - \hat{X}_t, \nonumber \\
    & = \langle c_a, z_t - \hat{z}_t \rangle + \eta_t, \nonumber \\
    & \sim \mathcal{N}(0,c_a^\top P c_a + \sigma_\eta). \nonumber
\end{align}

Note that since $\Gamma -\Gamma K C_\theta$ is Schur by construction, then the magnitude of the term $\langle c_a, (\Gamma -\Gamma K C_\theta)^s \hat{z}_{t-s}\rangle$ decreases as $s$ increases. If given a set of time instants $\mathcal{T}_{a} = \{t_1,\dots,t_{N_a}\}$
then \eqref{eq:Matrix_Form} can be rearranged in the following form 
% \begin{equation}\label{eq:difference_form}
%     \sum_{\tau \in \mathcal{T}_a} X_{\tau} = \sum_{\tau \in \mathcal{T}_a} G_{a}\Theta_{\tau} + \langle c_a, (\Gamma -\Gamma K C_\theta)^s \hat{z}_{\tau-s}\rangle + \varepsilon_{a;\tau}. 
% \end{equation}
\begin{multline}\label{eq:difference_form}
    \begin{bmatrix}
        X_{t_1}^\top\\
        \vdots \\
        X_{t_{N_a}}^\top
    \end{bmatrix}^\top = G_a \begin{bmatrix}
        \Theta_{t_1}^\top \\
        \vdots \\
        \Theta_{t_{N_a}}^\top
    \end{bmatrix}^\top \\ + \begin{bmatrix}
    \langle c_a, (\Gamma -\Gamma K C_\theta)^s \hat{z}_{t_1-s}\rangle^\top + \varepsilon_{a;t_1}^\top \\
    \vdots \\
    \langle c_a, (\Gamma -\Gamma K C_\theta)^s \hat{z}_{t_{N_a}-s}\rangle^\top + \varepsilon_{a;t_{N_a}}^\top
    \end{bmatrix}^\top. 
\end{multline}

A regularized least squares estimate for \eqref{eq:difference_form} is
\begin{align}
    \hat{G}_{a} & = \begin{bmatrix}
        X_{t_1}^\top\\
        \vdots \\
        X_{t_{N_a}}^\top
    \end{bmatrix}^\top \begin{bmatrix}
        \Theta_{t_1}^\top \\
        \vdots \\
        \Theta_{t_{N_a}}^\top
    \end{bmatrix} \left( \lambda I + \begin{bmatrix}
        \Theta_{t_1}^\top \\
        \vdots \\
        \Theta_{t_{N_a}}^\top
    \end{bmatrix}^\top \begin{bmatrix}
        \Theta_{t_1}^\top \\
        \vdots \\
        \Theta_{t_{N_a}}^\top
    \end{bmatrix}
    \right)^{-1}, \nonumber \\
    & = \sum_{\tau \in \mathcal{T}_a} X_\tau \Theta_{\tau}^\top V_{a}^{-1}. \label{eq:identify}
\end{align}
where $V_a$ is defined to be
\begin{equation}
    V_{a} \triangleq \lambda I + \sum_{\tau \in \mathcal{T}_a} \Theta_{\tau} \Theta_{\tau}^\top, \label{eq:identify_V}
\end{equation}
and $\lambda \geq 0$ is the regularization parameter. Note that the $\lambda I$ is added so that that $V_{a}$ is positive definite and therefore invertible. 

\section{Bandit Strategy}\label{sec:Bandit_Strategy}

The action the learner ought to choose is the action that the learner predicts will output the highest reward. This prediction is based on the matrix $\hat{G}_a$, which is an estimate of the matrix $G_a$. Therefore, the learner should focus on learning $\hat{G}_a$ for each action $a \in \{1,\dots,k\}$ at the beginning and then choose an action based on $\hat{G}_a$ after the initial phase. 

The proposed strategy the learner will use is the following. The parameter to set is $s$, where a larger $s$ value decreases the bias term $\langle c_a, (\Gamma -\Gamma K C_\theta)^s \hat{z}_{t-s}\rangle$ in \eqref{eq:Matrix_Form} which impacts the identification error $||\hat{G}_a - G_a||_2$. At the start of the algorithm, the learner will cycle through each action $a \in \{1,\dots,k\}$ round $t = 1$ to round $t = ks$. The learner will start computing $\hat{G}_a$ once $t > s$. After round $t = ks$, the learner will choose the action $a \in \{1,\dots,k\}$ that has the largest $\hat{G}_a \Theta_t$ value and update $\hat{G}_a$. 

The proposed Algorithm \ref{alg:cap}, inspired by Explore-Then-Commit \cite{lattimore2020bandit}, is summarized below.   
\begin{algorithm}
    \caption{Systems-Based Explore-Then-Exploitation (\textbf{SB-ETC})}\label{alg:cap}
    \begin{algorithmic}
    \STATE $t \gets 1$
    \FOR{$a \in \{1,\dots,k\}$}
        \STATE $S_a \gets 0$
        \STATE $\hat{G}_{a} \gets \mathbf{0}_{1 \times ms+1}$
        \STATE $\mathcal{T}_a \gets \{\}$
    \ENDFOR
    \WHILE{ $t \leq n$ }
        \IF{ $t > ks$ }
        \STATE $a \gets \underset{a\in \{1,\dots,k\}}{\arg\max} \textbf{ } \hat{G}_{a} \Theta_{t}$
        \ELSE
        \STATE $a \gets \underset{a\in \{1,\dots,k\}}{\arg\max} \textbf{ } \frac{1}{S_a}$
        \STATE $S_a \gets S_a + 1$
        \ENDIF
        \IF{ $t > s$ }
        \FOR{ $a \in \{1,\dots,k\}$ }
        \STATE Update $\mathcal{T}_a$
        \STATE Learn $\hat{G}_{a}$ based on \eqref{eq:identify} and \eqref{eq:identify_V}
        \ENDFOR
        \ENDIF
        \STATE Sample $(\theta_t,X_t)$ based on \eqref{eq:Linear_System}
        \STATE $t \gets t + 1$
    \ENDWHILE
\end{algorithmic}
\end{algorithm}

% \begin{remark}
%     It should be noted that the length of the exploration round for this paper is $ks$. Future research will focus on finding a more efficient exploration length. 
% \end{remark}

\section{Regret Analysis of \textbf{SB-ETC}}\label{sec:Regret}
The following theorem below provides a bound for regret defined in \eqref{eq:pseudo-regret}. 

\begin{theorem}\label{theorem:regret}
    Given a failure rate of $\delta \in (0,1)$, regret as in \eqref{eq:pseudo-regret} has the following bound with a probability of at least $1-\delta$: 
    \begin{multline}
        R_n \leq \sum_{t=1}^{ks}\sum_{a \neq a^*} \mathbb{E}[\langle \Delta c_a, z_t \rangle + \Delta \mu_a] \\ + \sum_{t=ks+1}^n\sum_{a \neq a^*} \mathbb{E}_{z}[\langle \Delta c_a, z_t \rangle + \Delta \mu_a|a]\cdot \\\min\Bigg\{B_a\mathbb{E}_{\Theta}\Bigg[\frac{||\Theta_{t}||_2}{|\Delta G_{a}\Theta_{t}|}\Bigg],1\Bigg\}, \label{eq:regret_bound}
    \end{multline}
    where $\Delta c_a$, $\Delta \mu_a$ and $\Delta G_{a}$ are defined to be 
    \begin{align}
        \Delta c_a & \triangleq c_{a^*} - c_a, \label{eq:delta_a}\\
        \Delta \mu_a & \triangleq \mu_{a^*} - \mu_{a}, \label{eq:delta_mu} \\
        \Delta G_{a} & \triangleq G_{a^*} - G_{a}. \label{eq:delta_G}
    \end{align}
    and $B_a$ is a bound such that with a probability of at least $1-\delta$:
    \begin{equation}\label{eq:event_1}
        ||\hat{G}_{a} - G_{a}||_2 + ||G_{a^*} - \hat{G}_{a^*}||_2 \leq B_a.
    \end{equation}
\end{theorem}

\begin{proof}
    Using the law of iterated expectations \cite{wooldridge2010econometric}, the instanteneous regret for one round $t = ks+1,\dots,n$ is 
    \begin{align}
        \mathbb{E}[X_t^* - X_t] & = \mathbb{E}[\langle \Delta c_a, z_t \rangle + \Delta \mu_a] \nonumber \\
        & = \mathbb{E}_{a}[\mathbb{E}_{z}[\langle \Delta c_a, z_t \rangle + \Delta \mu_a|a]] \nonumber \\
        & = \sum_{a = 1}^k \mathbb{E}_{z}[\langle \Delta c_a, z_t \rangle + \Delta \mu_a|a]\mathbb{P}[a], \label{eq:expectation_law}
    \end{align}
    where $\Delta c_a$ and $\Delta \mu_a$ are defined in \eqref{eq:delta_a} and \eqref{eq:delta_mu}. In the following, we will provide an upper bound for $\mathbb{P}[a]$. Consider the following event $\mathcal{E}_a$
    \begin{equation}\label{eq:event_a}
        \mathcal{E}_a \triangleq \{\hat{G}_{a} \Theta_{t} \geq \hat{G}_{a^*}\Theta_{t}\},
    \end{equation}
    which implies that modeling error leads to selecting an action rather than the optimal one. Note that $\mathbb{P}[a]$ is the following probability
    \begin{align}
        \mathbb{P}[a] & = \mathbb{P} \Big[\cap_{i \neq a} \{\hat{G}_{a} \Theta_{t} \geq \hat{G}_{i}\Theta_{t}\}\Big] \nonumber \\
        & = \mathbb{P} \Big[\cap_{i \neq a,a^*} \{\hat{G}_{a} \Theta_{t} \geq \hat{G}_{i}\Theta_{t}\} \cap \{\hat{G}_{a} \Theta_{t} \geq \hat{G}_{a^*}\Theta_{t}\}\Big],
    \end{align}
    implying that $\mathbb{P}[a] \leq \mathbb{P}[\mathcal{E}_a]$. Adjusting the inequality in $\mathcal{E}_a$ provides the following 
    \begin{align}
        \hat{G}_{a} \Theta_{t} & \geq \hat{G}_{a^*}\Theta_{t} \nonumber \\
        G_{a}\Theta_{t} + (\hat{G}_{a} - G_{a})\Theta_{t} & \geq G_{a^*}\Theta_{t} \nonumber \\
        & \quad\quad - (G_{a^*} - \hat{G}_{a^*})\Theta_{t}. \label{eq:event_inequality}
    \end{align}
    
    Let $\Delta G_{a}$ be as in \eqref{eq:delta_G}. Then \eqref{eq:event_inequality} can be written as follows: 
    \begin{align}
        \Delta G_{a}\Theta_{t} & \leq (\hat{G}_{a} - G_{a})\Theta_{t}
        \nonumber \\
        & \quad\quad + (G_{a^*} - \hat{G}_{a^*})\Theta_{t}, \nonumber 
    \end{align}
    which implies that
    \begin{align}
        \Delta G_{a}\Theta_{t}& \leq ||\hat{G}_{a} - G_{a}||_2||\Theta_{t}||_2 \nonumber \\
        & \quad\quad + ||G_{a^*} - \hat{G}_{a^*}||_2||\Theta_{t}||_2 . \label{eq:inequality_1}
    \end{align}
    
    At this stage, we use \eqref{eq:identify} and \eqref{eq:identify_V} to express $\hat{G}_a$ for any actions as follows:
    \begin{align}
        \hat{G}_{a} & = \sum_{\tau \in \mathcal{T}_{a}} G_{a}\Theta_{\tau} \Theta_{\tau}^\top V_{a}^{-1} \nonumber \\
        &\quad\quad + \langle c_a, (\Gamma -\Gamma K C_\theta)^s \hat{z}_{\tau-s}\rangle \Theta_{\tau}^\top V_{a}^{-1} \nonumber \\
        & \quad\quad + \epsilon_{a;\tau} \Theta_{\tau}^\top V_{a}^{-1}, \nonumber \\
        & = G_{a} (V_{a} - \lambda I) V_{a}^{-1}   \nonumber\\ 
        & \quad\quad + \sum_{\tau \in \mathcal{T}_{a}} \langle c_a, (\Gamma -\Gamma K C_\theta)^s \hat{z}_{\tau-s}\rangle \Theta_{\tau}^\top V_{a}^{-1} \nonumber \\ 
        & \quad\quad + \sum_{\tau \in \mathcal{T}_{a}} \epsilon_{a;\tau} \Theta_{\tau}^\top V_{a}^{-1}, \nonumber \\ 
        & = G_{a} - \lambda G_{a} V_{a}^{-1} \nonumber \\
        & \quad\quad + \sum_{\tau \in \mathcal{T}_{a}} \langle c_a, (\Gamma -\Gamma K C_\theta)^s \hat{z}_{\tau-s}\rangle \Theta_{\tau}^\top V_{a}^{-1} \nonumber \\ 
        & \quad\quad + \sum_{\tau \in \mathcal{T}_{a}} \epsilon_{a;\tau} \Theta_{\tau}^\top V_{a}^{-1}. \nonumber
    \end{align}
    
    Thus, the $\ell$-2 norm $||\hat{G}_{a} - G_{a}||_2$ is
    \begin{align}
        \hat{G}_{a} - G_{a} & = - \lambda G_{a}V_{a}^{-1} \nonumber \\ 
        & \quad + \sum_{\tau \in \mathcal{T}_{a}} \langle c_a, (\Gamma -\Gamma K C_\theta)^s \hat{z}_{\tau-s}\rangle\Theta_{\tau}^\top V_{a}^{-1}\nonumber \\ 
        & \quad + \sum_{\tau \in \mathcal{T}_{a}} \epsilon_{a;\tau}\Theta_{\tau}^\top V_{a}^{-1}, \nonumber \\
        ||\hat{G}_{a} - G_{a}||_2 & \leq ||\lambda G_{a}V_{a}^{-1}||_2 \nonumber \\ 
        & \quad + \Bigg|\Bigg|\sum_{\tau \in \mathcal{T}_{a}} \langle c_a, (\Gamma -\Gamma K C_\theta)^s \hat{z}_{\tau-s}\rangle\Theta_{\tau}^\top V_{a}^{-1}\Bigg|\Bigg|_2 \nonumber \\ 
        & \quad + \Bigg|\Bigg|\sum_{\tau \in \mathcal{T}_{a}} \epsilon_{a;\tau}\Theta_{\tau}^\top V_{a}^{-1}\Bigg|\Bigg|_2 . 
    \end{align}
    
    Since $\rho(\Gamma) < 1$, the estimate $\langle c_a, (\Gamma -\Gamma K C_\theta)^s \hat{z}_{\tau-s}\rangle$ is bounded. For the product $\epsilon_{a;t}\Theta_{t}^\top V_{a}^{-1}$, based on theorem 1 in \cite{NIPS2011_e1d5be1c}, since $\epsilon_{a;t}$ is conditionally $c_a^\top P c_a + \sigma_\eta$-sub-Gaussian and $\Theta_{t}$ is measurable, then given a failure rate $\delta \in (0,1)$, this term is bounded with a probability of at least $1-\delta$. Therefore, \eqref{eq:event_1} is satisfied with a probability of at least $1-\delta$. Denote \eqref{eq:event_1} as $\mathcal{E}_1$. Now assuming $\mathcal{E}_1$ is given, the inequality \eqref{eq:inequality_1} can be rewritten as
    \begin{align}
        B_a ||\Theta_{t}||_2 & \geq \Delta G_{a}\Theta_{t} \nonumber \\
        \Rightarrow \frac{||\Theta_{t}||_2}{\Delta G_{a}\Theta_{t}} & \geq \frac{1}{B_a} \nonumber \\
        \Rightarrow \frac{||\Theta_{t}||_2}{|\Delta G_{a}\Theta_{t}|} & \geq \frac{1}{B_a}. \label{eq:inequality_2}
    \end{align}
    
    Let \eqref{eq:inequality_2} be denoted as $\mathcal{E}_2|\mathcal{E}_1$. Based on the Markov inequality \cite{boucheron2013concentration}, the following concentration bound is given 
    \begin{equation}\label{eq:upper_e_a}
        \mathbb{P}[\mathcal{E}_2|\mathcal{E}_1] \leq \min\Bigg\{B_a\mathbb{E}_{\Theta}\Bigg[\frac{||\Theta_{t}||_2}{|\Delta G_{a}\Theta_{t}|}\Bigg],1\Bigg\}.
    \end{equation}
    
    Note \eqref{eq:event_a} and \eqref{eq:inequality_2} can be rewritten the following way
    \begin{align}
        \mathcal{E}_2|\mathcal{E}_1 & = \Big\{B_a \frac{||\Theta_{t}||_2}{|\Delta G_{a}\Theta_{t}|} \geq 1 \Big\}, \nonumber \\
        \mathcal{E}_a & = \Big\{\frac{(\hat{G}_{a} - G_{a})\Theta_{t} + (G_{a^*} - \hat{G}_{a^*})\Theta_t}{\Delta G_a\Theta_t} \geq 1\Big\}. 
    \end{align}
    
    Note that $\mathbb{P}[\mathcal{E}_a] \leq \mathbb{P}[\mathcal{E}_2|\mathcal{E}_1]$ is true if the following inequality is true. 
    \begin{equation}
        B_a \frac{||\Theta_{t}||_2}{|\Delta G_{a}\Theta_{t}|} \geq\frac{(\hat{G}_{a} - G_{a})\Theta_{t} + (G_{a^*} - \hat{G}_{a^*})\Theta_t}{\Delta G_a\Theta_t}. \label{eq:events}
    \end{equation}
    
    Since \eqref{eq:events} is satisfied with a probability of at least $1-\delta$ based on \eqref{eq:event_1}, then $\mathbb{P}[\mathcal{E}_a] \leq \mathbb{P}[\mathcal{E}_2|\mathcal{E}_1]$ with a probability of at least $1-\delta$. Therefore, \eqref{eq:expectation_law} has the following upper bound
    \begin{align}
        \mathbb{E}[X_t^* - X_t] & = \sum_{a=1}^k \mathbb{E}_z[\langle \Delta c_a,z_t \rangle + \Delta \mu_a |a] \mathbb{P}[a] \nonumber \\
        & \leq \sum_{a\neq a^*} \mathbb{E}_z[\langle \Delta c_a,z_t \rangle + \Delta \mu_a |a] \mathbb{P}[\mathcal{E}_a] \nonumber \\
        & \leq \sum_{a\neq a^*} \mathbb{E}_z[\langle \Delta c_a,z_t \rangle + \Delta \mu_a |a] \mathbb{P}[\mathcal{E}_2|\mathcal{E}_1]. \nonumber
    \end{align}

    Using \eqref{eq:upper_e_a}, this provides the upper-bound for \eqref{eq:expectation_law} with a probability of at least $1-\delta$
    \begin{multline}\label{eq:instant_regret_bound}
        \mathbb{E}[X_t^* - X_t] \leq \sum_{a \neq a^*} \mathbb{E}_{z}[\langle \Delta c_a, z_t \rangle + \Delta \mu_a |a]\cdot \\\min\Bigg\{B_a\mathbb{E}_{\Theta}\Bigg[\frac{||\Theta_{t}||_2}{|\Delta G_{a}\Theta_{t}|}\Bigg],1\Bigg\}.
    \end{multline}
    
    Summing \eqref{eq:instant_regret_bound} over $t$ rounds gives the bound \eqref{eq:regret_bound}.
    
\end{proof}

Based on theorem \ref{theorem:regret}, the regret performance is based on the bound for model error $B_a$. Therefore, if the model is known, then $B_a=0$ sets the bound to be zero after $t = ks$ which is reasonable. Note that $B_a$ is based on the number of times action $a$ is chosen, which can affect the upper bound. For now, the exploration period is set to $ks$ so that the learner has $s$ samples for each action $a \in \{1,\dots,k\}$. Future work will focus on what is a more effective length for the exploration period. 

\section{Application to trading}\label{sec:Numerical_Simulation}

This section will exemplify the use of the proposed framework. Let there be two stocks, $i = 1,2$ a trader is interested in. The trader can either buy then sell either stock 1 or 2, or refrain from trading for each round $t = 1,\dots,10^4$. Figure \ref{figure:strategy} provides an example timeline of the trader's strategy. 
\begin{figure}[thpb]
    \centering
    \includegraphics[width=0.4\textwidth]{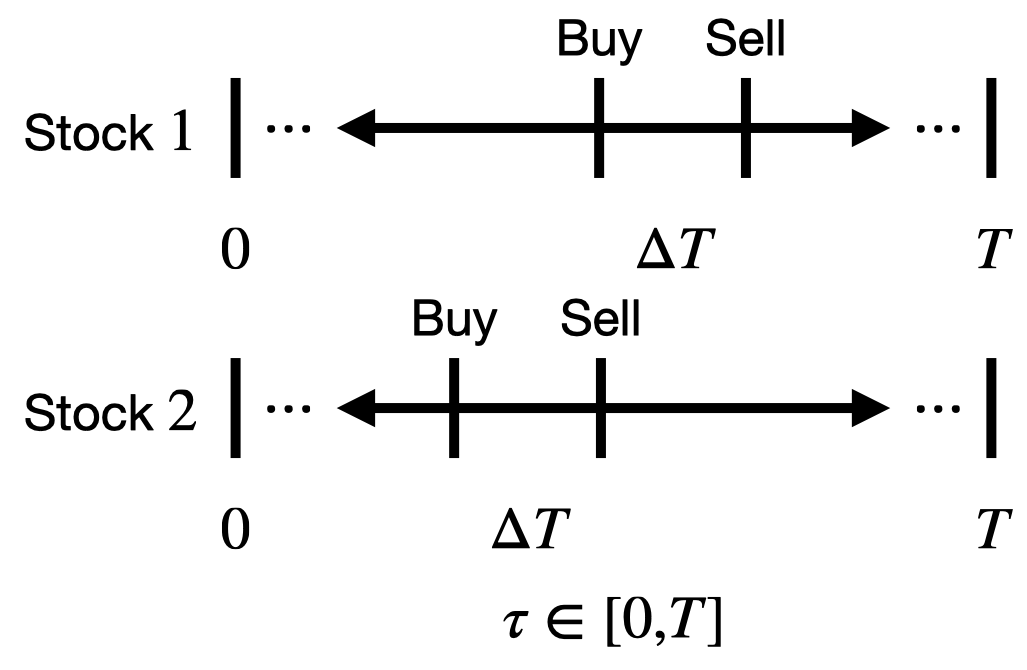}
    \caption{Example timeline of trader's strategy. }
    \label{figure:strategy}
\end{figure}
The context $\theta_t$ represent the price change for the stock. The reward $X_t$ is the financial gain (loss) deriving from trading a stock. Both variables are sampled from the following stochastic linear dynamical system, where its derivation is in the Appendix. 
\begin{equation}\label{eq:stock_model_dynamical_system}
    \begin{array}{ll}
        z_{t+1}' & = \begin{bmatrix}
            0 & 0 & 0 & 0 \\
            0 & 0 & 0 & 0 \\
            0 & 0 & 0.9512 & 0 \\
            0 & 0 & 0 & 0.6065 \\
        \end{bmatrix} z_t' + \xi_t \\
        \theta_t & =  \begin{bmatrix}
            -1 & 0 & 0.0353 & 0 \\
            0 & -1 & 0 & 0.2987
        \end{bmatrix} z_t' \\
        X_t & = \langle c_a, z_t' \rangle
    \end{array}, 
\end{equation}
\begin{align}
    \xi_t & \sim \mathcal{N}\left(0,\begin{bmatrix}
        0.9672 & 0 & 20.0957 & 0 \\
        0 & 0.6503 & 0 & 0.7536 \\
        20.0957 0 & & 1112.3 & 0 \\
        0 & 0.7536 & 0 & 4.0537
    \end{bmatrix}\right), \nonumber \\
    c_a & \in \{
        \begin{bmatrix}
            -1 & 0 & 0.0353 & 0
        \end{bmatrix},  \begin{bmatrix}
            0 & -1 & 0 & 0.2987
        \end{bmatrix}, \nonumber \\
        & \quad\quad\quad\quad\quad\quad\quad\quad\quad\quad\quad\quad\quad\quad \begin{bmatrix}
            0 & 0 & 0 & 0
        \end{bmatrix}\}. \nonumber 
\end{align}

For the model, the number of previous contexts $\theta_t$ that are used is $s = 10$ and the regularization parameter $\lambda = 10^{-1}$. Since $\Gamma$ is Schur, one could argue that the reward $X_t$ is just Gaussian distributed as $t >> 0$. Therefore, using \textbf{UCB} is a reasonable method to use as a comparison. The parameter used in \textbf{UCB} is $\delta = 0.1$. To provide an upper bound on the algorithm's performance we use an oracle that leverages optimal predicted awards generated by the Kalman filter $\hat{X}_{t|t-1}$ \eqref{eq:Kalman_Filter} to choose the action $c_a \in \mathcal{A}$. 

Figure \ref{figure:performance_1} contains the instanteneous regret (the top plot) and regret (the bottom plot) of the learner using \textbf{UCB} (the red line), algorithm \ref{alg:cap} (the blue line), and the oracle (the green line) averaged across 1,000 different simulations. Even though the system converges to the steady-state, \textbf{UCB} still performs sub-par compared to algorithm \ref{alg:cap}. Figure \ref{figure:performance_2} is a comparison of just algorithm \ref{alg:cap} and the oracle. It can be observed that algorithm \ref{alg:cap} instantenous regret is converging to the oracle's instanteneous regret. 

\begin{figure}[thpb]
    \centering
    \includegraphics[width=0.5\textwidth]{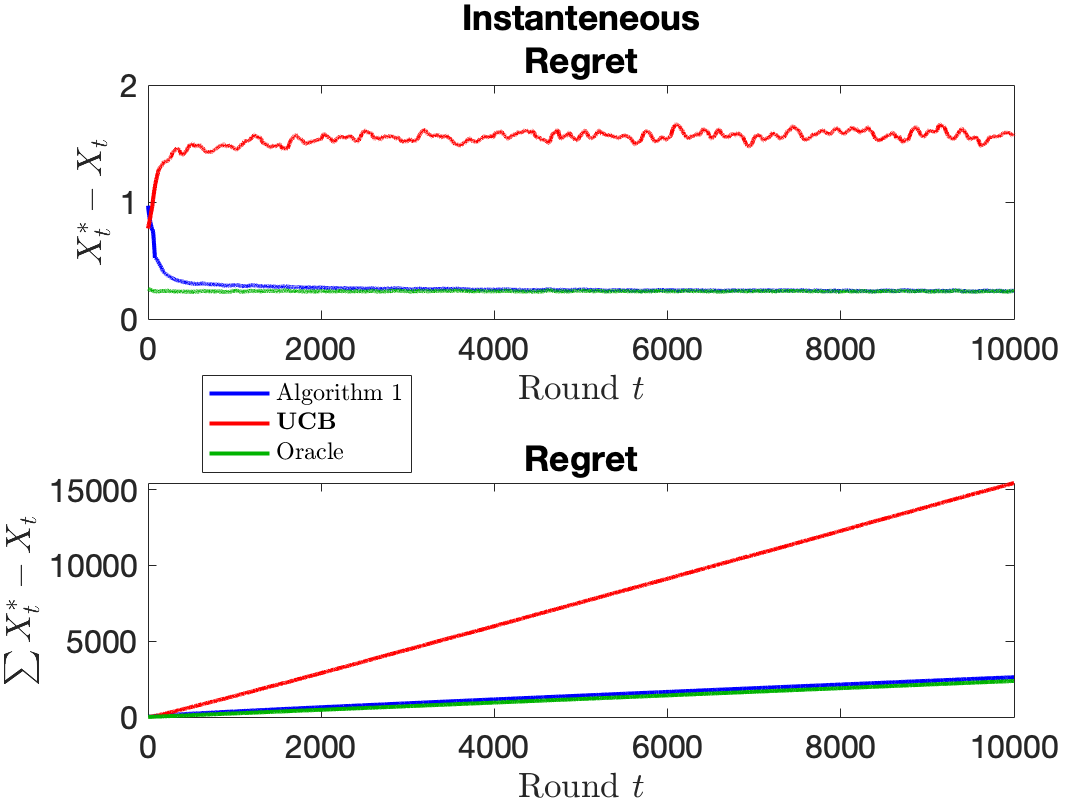}
    \caption{Figure for algorithm performance for \textbf{UCB}, algorithm \ref{alg:cap}, and the oracle. }
    \label{figure:performance_1}
\end{figure}
\begin{figure}[thpb]
    \centering
    \includegraphics[width=0.5\textwidth]{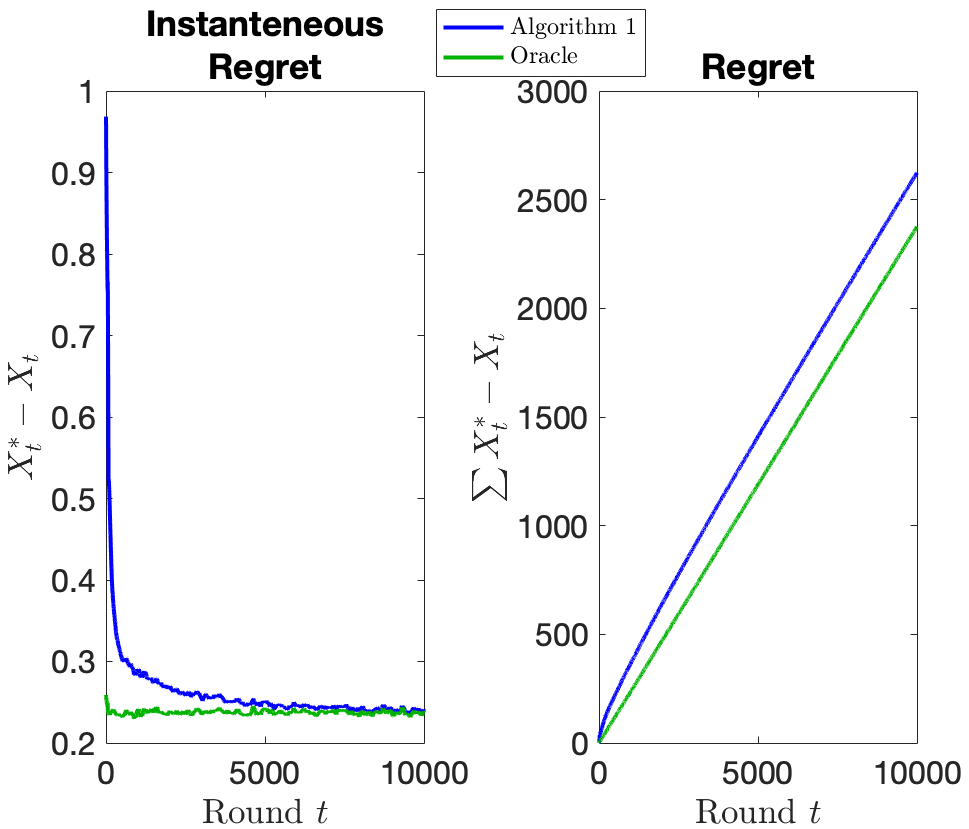}
    \caption{Figure for algorithm performance for algorithm \ref{alg:cap} and the oracle. }
    \label{figure:performance_2}
\end{figure}

\section{Conclusion}\label{sec:Conclusion}

This paper introduces a new variant of the stochastic multi-armed bandit where the rewards are sampled from an unknown stochastic linear dynamical system. To approach this problem, the learner first explores all the actions to learn the underlying linear model. Using the learned model, the learner uses the model to predict the action with the highest reward. Simulation results show how the proposed strategy yields near optimal actions after the learning phase. An application to high frequency trading is used to illustrate the methods and the results.

\bibliographystyle{IEEEtran}
\bibliography{IEEEabrv,autosam}{}

\appendix 

\section{Derivation for \eqref{eq:stock_model_dynamical_system}}

The trader models the price evolution of stock $i$ (denoted as $S_{\tau}^{[i]}$) using the following stochastic differential equation for $\tau \in [0,T]$ based on \cite{merton1973intertemporal}.
\begin{equation}
    \begin{array}{ll}
        \frac{dS_{\tau}^{[i]}}{S_{\tau}^{[i]}} & = M_{\tau}^{[i]} d\tau + dW_{\tau}^{[i]} \\
        dM_\tau^{[i]} & = \kappa_{[i]}\Big(\frac{1}{2} - M_\tau^{[i]}\Big)d\tau + \sigma_{[i]}dV_\tau^{[i]}
    \end{array},
\end{equation}
where $\kappa_{[i]}$ and $\sigma_{[i]}$, $i = 1,2$, are defined to be
\begin{align}
    \begin{bmatrix}
        \kappa_{[1]} \\
        \kappa_{[2]}
    \end{bmatrix} & \triangleq \begin{bmatrix}
        10^{-1} \\
        1
    \end{bmatrix}, \nonumber \\
    \begin{bmatrix}
        \sigma_{[1]} \\
        \sigma_{[2]}
    \end{bmatrix} & \triangleq \begin{bmatrix}
        10 \\
        1
    \end{bmatrix}. 
\end{align}

The variable $M_\tau^{[i]}$ is the drift rate of stock $i$, $\kappa_{[i]}$ is the speed of reversion (the rate $M_{\tau}^{[i]}$ returns to its mean), and $\sigma_{[i]}$ sets the magnitude of $dV_\tau^{[i]}$. Both $dW_{\tau}^{[i]}$ and $dV_{\tau}^{[i]}$ are independent Gaussian distributed random variables with a variance of $1$ with no time correlation, i.e. 
$dW_{\tau}^{[i]}\sim \mathcal{N}(0,\delta_\tau)$ and $dV_{\tau}^{[i]}\sim \mathcal{N}(0,\delta_\tau)$ where $\delta_\tau$ is the delta dirac function. Let $Y_\tau^{[i]} = \log(S_{\tau}^{[i]})$. Using It\^o's lemma, the stochastic differential equation for $\log(S_{\tau}^{[i]})$ is 
\begin{align}
    d\log(S_{\tau}^{[i]}) & = \Bigg\langle \frac{\partial \log(S_{\tau}^{[i]})}{\partial S_{\tau}^{[i]}},S_{\tau}^{[i]} M_{\tau}^{[i]} \Bigg\rangle d\tau \nonumber \\
    & \quad \quad \quad \quad + \frac{1}{2}\Bigg\langle \frac{\partial^2 \log(S_{\tau}^{[i]})}{\partial (S_{\tau}^{[i]})^2} , (S_{\tau}^{[i]})^2\Bigg\rangle d\tau \nonumber \\
    & \quad \quad \quad \quad + \Bigg\langle \frac{\partial \log(S_{\tau}^{[i]})}{\partial S_{\tau}^{[i]}}, S_{\tau}^{[i]} \Bigg\rangle dW_{\tau}^{[i]}, \\
    d\log(S_{\tau}^{[i]}) & = \frac{1}{S_{\tau}^{[i]}} S_{\tau}^{[i]} M_{\tau}^{[i]}d\tau + \Bigg(\frac{1}{2}\Bigg)\Bigg(\frac{-1}{(S_{\tau}^{[i]})^2}\Bigg)(S_{\tau}^{[i]})^2 d\tau \nonumber \\
    & \quad \quad \quad \quad + \frac{1}{S_{\tau}^{[i]}} S_{\tau}^{[i]} dW_{\tau}^{[i]}, \\
    d\log(S_{\tau}^{[i]}) & = \Big(M_{\tau}^{[i]} - \frac{1}{2}\Big)d\tau + dW_{\tau}^{[i]}. 
\end{align}

This leads to the following stochastic differential equations:
\begin{equation}
    \begin{array}{ll}
        dY_{\tau}^{[i]} & = \Big(M_{\tau}^{[i]} - \frac{1}{2}\Big) d\tau + dW_{\tau}^{[i]} \\
        dM_\tau^{[i]} & = \kappa_{[i]}\Big(\frac{1}{2} - M_\tau^{[i]}\Big)d\tau + \sigma_{[i]}dV_\tau^{[i]}
    \end{array}.
\end{equation}

Let the following matrices and vectors be defined as below:
\begin{align}
    y(\tau) & \triangleq \begin{bmatrix} Y_{\tau}^{[1]} & Y_{\tau}^{[2]}\end{bmatrix}^\top, \nonumber\\
    m(\tau) & \triangleq \begin{bmatrix} M_{\tau}^{[1]} & M_{\tau}^{[2]}\end{bmatrix}^\top, \nonumber\\
    \kappa & \triangleq \begin{bmatrix}
        \kappa_{[1]} & 0 \\
        0 & \kappa_{[2]}
    \end{bmatrix}, \nonumber\\
    F & \triangleq \left[ \begin{array}{c|c}
      \mathbf{0} & I_2 \\
      \midrule
      \mathbf{0} & - \kappa \\
    \end{array}\right], \nonumber
\end{align}
\begin{align}    
    dw(\tau) & \triangleq  \begin{bmatrix}
        dW_{\tau} \\
        \sigma dV_{\tau}
    \end{bmatrix} \sim \mathcal{N}(0,\Sigma \delta_\tau), \nonumber\\
    \Sigma & \triangleq \begin{bmatrix}
        1 & 0 & 0 & 0 \\
        0 & 1 & 0 & 0 \\
        0 & 0 & \sigma_{[1]}^2 & 0\\
        0 & 0 & 0 & \sigma_{[2]}^2
    \end{bmatrix} , \nonumber\\
    B_\mu & \triangleq \begin{bmatrix}
        -1/2 \\
        -1/2 \\
        \kappa_{[1]}/2\\
        \kappa_{[2]}/2
    \end{bmatrix}.
\end{align}

This provides the stochastic linear dynamical system 
\begin{equation}\label{eq:system_continuous}
    \begin{bmatrix}
        dy(\tau) \\
        dm(\tau)
    \end{bmatrix} = 
    \Big(F \begin{bmatrix}
        y(\tau) \\
        m(\tau)
    \end{bmatrix} + B_\mu\Big) dt + dw(\tau). 
\end{equation}

System \eqref{eq:system_continuous} is discretized with intervals of size $\Delta T$ ($n\Delta T = T$) which gives the following discrete-time stochastic linear dynamical system 
\begin{multline} \label{eq:discretized_system}
    \begin{bmatrix}
        y(t\Delta T + \Delta T) \\
        m(t\Delta T + \Delta T)
    \end{bmatrix} = \exp(F\Delta T)\begin{bmatrix}
        y(t\Delta T) \\
        m(t\Delta T)
    \end{bmatrix} \\ + \Delta B_\mu + \Delta w(t \Delta T) ,  
\end{multline}
where $\Delta B_\mu$ and $\Delta w(t\Delta T)$ are defined below. 
\begin{align}
    \Delta B_\mu & \triangleq \sum_{i = 0}^{\infty} F^i B_\mu \frac{\Delta T^{i+1}}{(i+1)!}, \label{eq:approximate_mu} \\
    \Delta w(t\Delta T) & \sim \mathcal{N}(0,\Xi), \nonumber \\
    \Xi & = \int_{t\Delta T}^{(t+1)\Delta T} e^{F((t+1)\Delta T-\tau)}S e^{F^\top(t+1)\Delta T-\tau)} d\tau. \label{eq:integral_S}
\end{align}

Equations \eqref{eq:approximate_mu} and \eqref{eq:integral_S} are from \cite{gelb1974applied}. Evaluating \eqref{eq:integral_S} is analytically intractable; therefore, \cite{van1978computing} is used to approximate $\Xi$. 
\begin{align}
    \left[ \begin{array}{c|c}
      \Phi_{1,1} & \Phi_{1,2} \\
      \midrule
      \mathbf{0} & \Phi_{2,2} \\
    \end{array}\right] & = \exp\left(\left[ \begin{array}{c|c}
      -F^\top & \Sigma \\
      \midrule
      \mathbf{0} & F \\
    \end{array}\right]\Delta T\right), \nonumber \\
    \Phi_{2,2}^\top \Phi_{1,2} & \approx \Xi. \label{eq:approximate_Q}
\end{align}

Say that the trader uses the following strategy: the trader buys stock $i$ at the start of time $t\Delta T - \Delta T$ and then sells that stock at time $t\Delta T$. Define reward $X_t$ for this time period to be
\begin{equation}
    X_t \triangleq Y_{t\Delta T}^{[i]} - Y_{(t-1)\Delta T}^{[i]} = \log\Big(\frac{S_{t\Delta T}^{[i]}}{S_{(t-1)\Delta T}^{[i]}}\Big). 
\end{equation}

Therefore, the difference $Y_{t\Delta T}^{[i]} - Y_{(t-1)\Delta T}^{[i]}$ is the logarithm of the percentage increase/decrease of buying at $t\Delta T- \Delta T$ and then selling at $t\Delta T $. We extend \eqref{eq:discretized_system} by using the following matrices and vectors:
\begin{align}
    z_t & \triangleq \begin{bmatrix}
        y(t\Delta T) \\
        m(t\Delta T) \\
        y(t\Delta T - \Delta T)
    \end{bmatrix} \in \mathbb{R}^{6}, \nonumber\\
    \Gamma & \triangleq 
     \left[ \begin{array}{c|c}
      F & \mathbf{0} \\
      \midrule
      I_2 & \mathbf{0}
    \end{array}\right], \nonumber \\
    \mu_\xi & \triangleq \begin{bmatrix}
        \Delta B_\mu^\top & \mathbf{0}^\top
    \end{bmatrix}^\top, \nonumber\\
    \xi_t & \sim \mathcal{N}(\mu_\xi,Q), \nonumber\\
    Q & \triangleq \left[ \begin{array}{c|c}
      \Phi_{2,2}^\top \Phi_{1,2} & \mathbf{0}\\
      \midrule
      \mathbf{0} & \mathbf{0} \\
    \end{array}\right], \nonumber\\
    C_\theta & \triangleq \begin{bmatrix}
        1 & 0 & 0 & 0 & -1 & 0 \\
        0 & 1 & 0 & 0 & 0 & -1
    \end{bmatrix}, \nonumber\\
    c_a & \in \{\begin{bmatrix}
        1 & 0 & 0 & 0 & -1 & 0
    \end{bmatrix}^\top, \nonumber \\
    & \quad\quad\quad\quad \begin{bmatrix}
        0 & 1 & 0 & 0 & 0 & -1
    \end{bmatrix}^\top, \nonumber \\
    & \quad\quad\quad\quad \begin{bmatrix}
        0 & 0 & 0 & 0 & 0 & 0
    \end{bmatrix}^\top\}.
\end{align}

Finally, the eigenvectors $\begin{bmatrix}
    U & U'
\end{bmatrix}$ of $\Gamma$ are computed, where $(C_\theta U,U^{-1}\Gamma U)$ is observable and $U^{-1}\Gamma U$ is Schur. This transformation provides \eqref{eq:stock_model_dynamical_system}.

\balance
\end{document}